\documentclass[conference]{IEEEtran}

\IEEEoverridecommandlockouts

\usepackage[cmex10]{amsmath}

\renewcommand\IEEEkeywordsname{Keywords}

\usepackage[utf8]{inputenc}
\usepackage[english]{babel}
\usepackage[T1]{fontenc}
\usepackage{amsfonts}
\usepackage{amsmath}
\usepackage{amsthm}
\usepackage{amssymb}
\usepackage{graphicx}
\usepackage{graphics}

\usepackage{bm}

\usepackage{nicefrac}    
\usepackage{microtype}   
\usepackage{booktabs}    

\let\originalleft\left
\let\originalright\right
\renewcommand{\left}{\mathopen{}\mathclose\bgroup\originalleft}
\renewcommand{\right}{\aftergroup\egroup\originalright}

\usepackage{pifont}
\usepackage{comment}
\usepackage{todonotes}
\usepackage[normalem]{ulem}
\usepackage{etoolbox}
\usepackage{color}
\newtoggle{FinalVersion}
\toggletrue{FinalVersion}
\iftoggle{FinalVersion}{
 
 \newcommand{\deleted}[1]{{}}
}{
 
 \newcommand{\deleted}[1]{{\sout{\color{red}#1}}} 
}

\usepackage[section]{algorithm}
\usepackage{algorithmicx}
\usepackage{algpseudocode}

\algdef{SE}[DOWHILE]{Do}{doWhile}{\algorithmicdo}[1]{\algorithmicwhile\ #1}%

\newtheorem{thm}{Theorem}[section]

\newtheorem{lemma}[thm]{Lemma}
\theoremstyle{definition}

\newcommand{\norm}[1]{\|#1\|}
\newcommand{\nrm}[1]{|#1|}

\newcommand{\R}{{\mathbb R}}

\newcommand{\bpm}{\begin{pmatrix}}
\newcommand{\epm}{\end{pmatrix}}

\newcommand{\mxmz}{\operatorname*{maximize}}

\newcommand{\argmax}{\operatorname*{argmax}}
\newcommand{\eps}{{\varepsilon}}

\newcommand{\neval}{N_{\rm eval}}
\newcommand{\nloc}{N_{\rm loc}}
\newcommand{\mA}{\textit{Mesh}}
\newcommand{\mB}{\textit{Time-optimal}}
\newcommand{\mC}{\textit{Robust}}

\usepackage{tikz}
\usepackage{pgfplots}
\usepackage{pgfplotstable}

\newtoggle{Draft}
\togglefalse{Draft}
\iftoggle{Draft}{
 \excludecomment{figure}
 
}{

}

\usetikzlibrary{pgfplots.groupplots}
\pgfplotsset{compat=1.3}
\pgfplotsset{
    line1/.style={%
        black, dashed, no marks, thick},
    line2/.style={%
        blue, no marks, thick},
    line2a/.style={%
        black, dotted, thick},
    line3a/.style={%
        blue, thick},
    line4a/.style={%
        orange, dashed, thick},
    row1/.style={%
        xmin=0, xmax=21,
        ymin=35, ymax=70},
    row2/.style={%
        xmin=0, xmax=31,
        ymin=0, ymax=12},
    row3/.style={%
        xmin=0, xmax=101,
        ymin=0, ymax=18},
}

\pgfplotstableread[col sep = comma]{Results_Counts.csv}\tabCounts
\pgfplotstableread[col sep = comma]{Results_Max.csv}\tabMax
\pgfplotstableread[col sep = comma]{Results_Default1_Average.csv}\tabBenchAAverA
\pgfplotstableread[col sep = comma]{Results_Default1_Survival_Delta=40.csv}\tabBenchASurvAA
\pgfplotstableread[col sep = comma]{Results_Default1_Survival_Delta=50.csv}\tabBenchASurvAB
\pgfplotstableread[col sep = comma]{Results_Fu2013_Average.csv}\tabBenchAAverB
\pgfplotstableread[col sep = comma]{Results_Fu2013_Survival_Delta=40.csv}\tabBenchASurvBA
\pgfplotstableread[col sep = comma]{Results_Fu2013_Survival_Delta=50.csv}\tabBenchASurvBB

\pgfplotstableread[col sep = comma]{Results_Default2_Average.csv}\tabBenchBAverA
\pgfplotstableread[col sep = comma]{Results_Default2_Survival_Delta=40.csv}\tabBenchBSurvAA
\pgfplotstableread[col sep = comma]{Results_Default2_Survival_Delta=50.csv}\tabBenchBSurvAB
\pgfplotstableread[col sep = comma]{Results_Fu2015_Average.csv}\tabBenchBAverB
\pgfplotstableread[col sep = comma]{Results_Fu2015_Survival_Delta=40.csv}\tabBenchBSurvBA
\pgfplotstableread[col sep = comma]{Results_Fu2015_Survival_Delta=50.csv}\tabBenchBSurvBB

\usepackage{textcomp}

\newcommand\copyrighttext{%
  \footnotesize \textcopyright 2019 IEEE. Personal use of this material is permitted. Permission from IEEE must be obtained for all other uses, in any current or future media, including reprinting/republishing this material for advertising or promotional purposes, creating new collective works, for resale or redistribution to servers or lists, or reuse of any copyrighted component of this work in other works.}
\newcommand\copyrightnotice{%
\begin{tikzpicture}[remember picture,overlay]
\node[anchor=south,yshift=10pt] at (current page.south) {\fbox{\parbox{\dimexpr\textwidth-\fboxsep-\fboxrule\relax}{\copyrighttext}}};
\end{tikzpicture}%
}

\begin{document}

\title{A Simple Yet Effective Approach to \\ Robust Optimization Over Time}

\author{
\IEEEauthorblockN{Luk\'a\v{s} Adam \qquad and \qquad Xin Yao}
\thanks{This work was supported by National Natural Science Foundation of China (Grant No. 61850410534), the Program for Guangdong Introducing Innovative and Enterpreneurial Teams (Grant No. 2017ZT07X386), Shenzhen Peacock Plan (Grant No. KQTD2016112514355531), and the Program for University Key Laboratory of Guangdong Province (Grant No. 2017KSYS008),

Both authors are with Shenzhen Key Laboratory of Computational Intelligence, University Key Laboratory of Evolving Intelligent Systems of Guangdong Province, Department of Computer Science and Engineering, Southern University of Science and Technology, Shenzhen 518055, China. Email: adam@utia.cas.cz, xiny@sustech.edu.cn (corresponding author)
}
}

\maketitle
\copyrightnotice

\begin{abstract}
Robust optimization over time (ROOT) refers to an optimization problem where its performance is evaluated over a period of future time. Most of the existing algorithms use particle swarm optimization combined with another method which predicts future solutions to the optimization problem. We argue that this approach may perform subpar and suggest instead a method based on a random sampling of the search space. We prove its theoretical guarantees and show that it significantly outperforms the state-of-the-art methods for ROOT.
\end{abstract}


\IEEEpeerreviewmaketitle

\section{Introduction}

Classical optimization problems involve minimizing or maximizing a function $f$ over a region $X$. Often, these problems depend on time $t$ and random variables (also called environments) $\bm\alpha(t)$. These problems may be written as
\begin{equation}\label{eq:problem_gen}
\mxmz_{x\in X} f(\bm x; \bm\alpha(t)).
\end{equation}
We focus on the case where at time $t$ only the history of $\bm\alpha(t)$ is known and where there is no information about its future distribution. Moreover, the objective may be accessed only via black-box evaluations without knowing the exact value of $\bm\alpha(t)$. The goal is to find the optimal solution to \eqref{eq:problem_gen}. Since the computation budget is limited, the solution at the current time should be found with the help of function evaluations at previous times.

The setting above describes the ``solution tracking'' where the solution may be recomputed and changed at every time instant. However, this is often not desirable or even impossible as a reimplementation of a solution may be physically impossible or may cause additional costs or inconvenience to users.

Another approach was proposed in \cite{yu2010robust} where the emphasis is not given to the performance up to the current time but over a future time period. Thus, the solution does not have to perform exceptionally well at present but it has to perform satisfactorily over time. The authors named this problem Robust optimization over time (ROOT).

A good ROOT solution should show a good performance in at least one of the two main performance criteria \cite{fu2015robust}. The first one is the average performance over a future time interval while the second one counts how long a solution performs better than a given threshold (precise definitions will be given later).

In this paper, we follow two goals. First, we propose a novel method. While the current state-of-the-art methods use a modification of particle swarm optimization, we propose to uniformly sample the search space and then improve the best point by a local search. The uniform search has the advantage that it gives theoretical bounds for the solution quality. Moreover, if the problem dimension is low, the sampled points may be the same for all time instants. This allows using prediction algorithms without having to reevaluate the functions at previous time instants.

Second, the ROOT papers usually did not describe the parameter initialization, boundary conditions or dynamics properly (see Section \ref{sec:overview}). They contain confusing notations and even plain mistakes. We conjecture that even though all ROOT papers used the same modified moving peak benchmark, they solved different problems due to different parameter settings. At the same time, the papers often did not propose a comparison with the basic benchmark: the solution which performs best at the current time and ignores the future. We try to remedy this situation by describing the benchmark properly, showing a proper comparison with a basic solution approach, and by providing our codes online so that any inconsistency can be immediately clarified.\footnote{\texttt{https://github.com/sadda/ROOT-Benchmark}}

The paper is organized as follows: the introduction is concluded by a short literature survey. In Section \ref{sec:algorithms} we propose our novel method and in Section \ref{sec:benchmarks} we try to codify the benchmark problems. Section \ref{sec:numerics} consists of the numerical part. To keep the paper as clear as possible, multiple results were moved to the Appendix.



\section{A simple approach to ROOT}\label{sec:algorithms}

In this section, we provide a literature overview, specify the problem formulation, propose a solution method and perform its basic analysis.

\subsection{Literature overview}\label{sec:overview}

There are numerous alternatives to approaching \eqref{eq:problem_gen}. Stochastic optimization \cite{birge2011introduction} maximizes $f$ in expectation while robust optimization \cite{ben2009robust} maximizes it in worst-case. Dynamic optimization \cite{bertsekas1995dynamic} models the evolution via an ordinary differential equation while multi-stage programming \cite{pereira1991multi} generalizes the stochastic optimization by considering a longer horizon. All of these fields assume the knowledge of the distribution of $\bm \alpha(t)$ and they are computationally rather expensive.

Concerning the literature overview for ROOT, \cite{yu2010robust} was the first paper to propose the ROOT problem. This paper did not consider any numerical results. \cite{fu2012characterizing} suggested new metrics requiring the knowledge of the optimal solution and tried to formalize the benchmark problem. \cite{fu2013finding} suggested the survival metric where the optimal solution does not need to be known. \cite{jin2013framework} investigated predicting the future by autoregressive series. \cite{guo2014find} considered ROOT as a bi-objective problem of maximizing the survival time and the average future fitness. \cite{fu2015robust} provided a new benchmark with known solutions. \cite{yazdani2017new} proposed a new method based on multi-swarm particle optimization. \cite{novoa2018approximation} investigated several methods for predicting future solutions. \cite{yazdani2018robust} proposed new techniques to predict future solutions and provided extensive literature overview and numerical study. \cite{chen2015evolutionary} generalized the concept into the multi-objective optimization.

\subsection{Problem formulation}

We consider the time discrete ROOT problem, where we need to solve \eqref{eq:problem_gen} for all $t\in\{1,\dots,T\}$. We consider a rather general case where at time $t$ we can evaluate the objective value $f(\bm x,\bm\alpha(t))$ for any query point $\bm x$. We do not know the exact value of $\bm\alpha(t)$ or its future distribution but we can make use of all queries (function evaluations) from previous time instants $1,\dots,t-1$.

To evaluate the solution $\bm x(t)$ quality at time $t$, we consider two metrics
\begin{equation}\label{eq:metrics}
\aligned
F_{\rm aver}(\bm x(t); t) &= \frac1S\sum_{s=0}^{S-1} f(\bm x(t); \bm\alpha(t+s)), \\
F_{\rm surv}(\bm x(t); t) &= \min\{s\ge 0 \mid f(\bm x(t); \bm\alpha(t+s)) \le f^* \}.
\endaligned
\end{equation}
The averaged objective metric $F_{\rm aver}$ measures the average from the future $S$ values while the survival metric $F_{\rm surv}$ measures how long the objective stays above a threshold $f^*$. Note that both metrics make use of the objective function $f$ at the current time (which can be evaluated) and at the future times (which can be only predicted).

A word of caution is needed here. The future values $\bm\alpha(t+s)$ in \eqref{eq:metrics} are considered to be fixed but not known. In the field of stochastic optimization \cite{birge2011introduction} this amounts to adding expectation with respect to $\bm\alpha$ to \eqref{eq:metrics}. Since in the numerical section, we will average the results with respect to different realizations of $\bm\alpha$, we should technically add this expectation to \eqref{eq:metrics} as well. The key difference is that stochastic optimization assumes the future distribution to be known while we assume it to be unknown.

\subsection{Proposed methods}

Most of the existing methods for ROOT are based on particle swarm optimization. These papers do not provide any convergence proofs and require hyperparameter tuning. In this section, we propose two very simple methods which do not suffer from these issues. The first one solves \eqref{eq:problem_gen} at the current time $t$ without considering the past or the future while the second one tries to obtain a robust solution. Note that at every time instant $t$, we have the computational budget of $\neval$ evaluations of $f(\cdot;\bm\alpha(t))$.

The first method spends $N$ evaluations on a global search and $\nloc=\neval-N$ evaluations on a local search. The global search is performed by a uniform discretization of the search space into $\{\bm x_1,\dots,\bm x_N\}$ and evaluating $f_n(t) = f(\bm x_n; \bm \alpha(t))$ for all $n=1,\dots,N$. Then we find the index $n_{\rm max}$ where $f_n(t)$ has the maximal value and improve $\bm x_{n_{\rm max}}$ by any local search method within $\nloc$ function evalutions. We provide a summary in Algorithm \ref{alg:method1}.

\begin{algorithm}[H]
    \centering
    \caption{Hybrid uniform sampling and local search method for solving ROOT}
    \label{alg:method1}
    \begin{algorithmic}[1]
\Require Number of function evaluation $\neval$, number of function evaluations for the local search $\nloc$
\State Set $N\gets \neval - \nloc$
\State Discretize the search space $X$ into $\bm x_1,\dots,\bm x_N$
\For{$t=1,\dots,T$}
\State Evaluate $f_n(t)\gets f(\bm x_n;\bm\alpha(t))$ for $n=1,\dots,N$
\State Find the index $n_{\rm max}$ with maximal value of $f_n(t)$
\State Improve $\bm x_{n_{\rm max}}$ by local search in $\nloc$ function evaluations to obtain optimal solution $\bm x_{\rm opt}(t)$
\EndFor
\State \textbf{return} $(\bm x_{\rm opt}(1), \dots, \bm x_{\rm opt}(T))$
    \end{algorithmic}
\end{algorithm}

The second method spends all $\neval$ evaluations on a global search. Again, we uniformly discretize the search space into $\{\bm x_1,\dots,\bm x_{\neval}\}$ and evaluate $f_n(t) = f(\bm x_n; \bm \alpha(t))$ for all $n=1,\dots,\neval$. The robust solution is selected by any method which takes into account the function values at a neighborhood or at previous time instants. Since the space discretization is the same at every time, besides $f_n(t)$ we also know $f_n(t-1),\dots,f_n(1)$ from previous iterations and we do need to invest any additional function evaluations. Thus, we may apply most of the methods from other ROOT papers for free. We provide a summary in Algorithm \ref{alg:method2}.

\begin{algorithm}[H]
    \centering
    \caption{Uniform sampling method for solving ROOT}
    \label{alg:method2}
    \begin{algorithmic}[1]
\Require Number of function evaluation $\neval$
\State Discretize the search space $X$ into $\bm x_1,\dots,\bm x_{\neval}$
\For{$t=1,\dots,T$}
\State Evaluate $f_n(t)\gets f(\bm x_n;\bm\alpha(t))$ for $n=1,\dots,\neval$
\State Based on $f_n(t),f_n(t-1),\dots$ for $n=1,\dots,\neval$ find robust solution $\bm x_{\rm rob}(t)$ 
\EndFor
\State \textbf{return} $(\bm x_{\rm rob}(1), \dots, \bm x_{\rm rob}(T))$
    \end{algorithmic}
\end{algorithm}

If the search space is $X=[x_{\rm min}, x_{\rm max}]^D$, then Appendix \ref{app:lemma} implies that the procedure from Algorithm \ref{alg:method1} gives a solution which is optimal with the following bound
\begin{equation}\label{eq:bound1}
f(\bm x_{\rm opt}(t);\bm\alpha(t)) \ge f^*(t) - \frac{L\sqrt{D}(x_{\rm max} - x_{\rm min})}{2(N^{\frac1D}-1)},
\end{equation}
where $f^*(t)$ is the optimal solution at time $t$ and $L$ is the so-called Lipschitz constant of $f(\cdot;\bm\alpha(t))$. Since most ROOT methods were tested for the two-dimensional case $D=2$, the previous bound is rather tight.
The solution quality is further improved by the local search.  

We would like to summarize the benefits of our approach:
\begin{enumerate}
\item Equation \eqref{eq:bound1} gives a guaranteed bound for the solution quality.
\item Since the same points are evaluated at all time instants, using any tracking or prediction mechanism from other ROOT papers requires no additional function evaluations.
\end{enumerate}

\section{Numerical benchmarks}\label{sec:benchmarks}

In this section, we describe the moving peak benchmark commonly used in the ROOT literature. It is based on \cite{branke1999memory} and appeared in many papers \cite{fu2015robust,fu2012characterizing,fu2013finding,jin2013framework,guo2014find,yazdani2017new,novoa2018approximation,yazdani2018robust}. However, to the best of our knowledge, no complete and proper description was given in any of these papers. Since, as we will show later, even a small change in the problem setting may have a large impact on the optimal solution, we try to provide a rigorous statement of the benchmark problems.

\subsection{Moving peaks benchmark 1}

This benchmark considers $M$ peaks of conic shape in $\R^D$. Peak $m$ has center $\bm c^m$, height $h^m$ and width $w^m$. Defining the random vector $\bm \alpha=(\bm c^m,h^m,w^m)_{m=1}^M$, the objective function
$$
f_t^1(\bm x;\bm \alpha(t)) = \max_{m=1,\dots,M}\left(h_t^m-w_t^m\norm{\bm x-\bm c_t^m}_{l_2}\right),
$$
measures that the height of maximal peak at $\bm x$. We use the shortened notation $h_t=h(t)$.

The dynamics of the random vector is given by
\begin{equation}\label{eq:bench1_dynamics}
\aligned
h_{t+1}^m &= h_t^m + \sigma_h^m \cdot N(0,1), \\
w_{t+1}^m &= w_t^m + \sigma_w^m \cdot N(0,1), \\
\bm c_{t+1}^m &= \bm c_t^m + \bm v_{t+1}^m, \\
\bm v_{t+1}^m &= s^m\frac{(1-\lambda)\bm r_{t+1}^m + \lambda \bm v_t^m}{\norm{(1-\lambda)\bm r_{t+1}^m + \lambda \bm v_t^m}}.
\endaligned
\end{equation}
Here, $N(0,1)$ denotes the normal distribution with zero mean and unit variance, $\bm r_t^m$ follows the uniform distribution on the $D$-dimensional sphere with radius $s^m$ and $\sigma_h^m\ge 0$, $\sigma_w^m\ge 0$ and $\lambda\in[0,1]$ are fixed parameters. The peak height $h_{t+1}^m$ differs from the previous height $h_t^m$ by a random number drawn from the normal distribution with zero mean and standard deviation $\sigma_h^m$. Similar holds true for the widths. The center $\bm c_{t+1}^m$ moves from $\bm c_t^m$ by vector $\bm v_{t+1}^m$. If $\bm v_1^m$ has norm $s^m$, then we have
$$
\aligned
\lambda = 0 &\implies \bm v_{t+1}^m = \bm r_{t+1}^m, \\
\lambda = 1 &\implies \bm v_{t+1}^m = \bm v_t^m.
\endaligned
$$
Thus, $\lambda=0$ implies that the movement of the peak centers is random while $\lambda=1$ implies that the movement is constant in direction $\bm v_1^m$. In both cases the distance between the previous and new centers is $s^m$.

The random variables have their bounds. We require $h_t^m\in [h_{\rm min}, h_{\rm max}]$ and $w_t^m\in [w_{\rm min}, w_{\rm max}]$. The bounds for the centers $\bm c_t^m\in [x_{\rm min}, x_{\rm max}]^D$ are the same as for the search space. If the dynamics \eqref{eq:bench1_dynamics} pushes some variable out of its corresponding bounds, we project (clip) it back. 

Finally, for initialization of \eqref{eq:bench1_dynamics} we need to know the initial centers $\bm c_0^m$, heights $h_0^m$, widths $w_0^m$ and the initial speeds $\bm v_0^m$. Following previous papers, we initialize the centers randomly in the search space $[x_{\rm min}, x_{\rm max}]^D$, the heights and widths to some known values $h_{\rm init}$ and $w_{\rm init}$, respectively and the initial speed is generated randomly at the $D$-dimensional sphere with radius $s^m$.

Note that in the literature there are some differences which we summarize in Appendix \ref{app:dynamics}.

\subsection{Moving peaks benchmark 2}

The second benchmark problem was defined in \cite{fu2015robust} by the objective
$$
\aligned
f_t^2(\bm x;\bm \alpha(t)) &= \frac1D\sum_{d=1}^D\max_{m=1,\dots,M}\left(h_t^{m,d}-w_t^{m,d}\nrm{x^d-c_t^{m,d}}\right), \\
\endaligned
$$
The upper index $d$ denotes the $d^{\rm th}$ component of a vector. Then the $D$-dimensional problem can be decomposed into $D$ one-dimensional problems. Moreover, since the heights are different in each dimension, the problem does not technically handle moving peaks anymore.

The authors in \cite{fu2015robust} considered several dynamics, we will mention only the one most similar to \eqref{eq:bench1_dynamics}, namely 
\begin{equation}\label{eq:bench2_dynamics}
\aligned
h_{t+1}^{m,d} &= h_t^{m,d} + \sigma_h^m \cdot N(0,1), \\
w_{t+1}^{m,d} &= w_t^{m,d} + \sigma_w^m \cdot N(0,1), \\
\bm c_{t+1}^m &= R(\theta_t^{D-1},\dots,\theta_t^1) \bm c_t^m, \\
\theta_{t+1}^d &= \theta_t^d + \sigma_\theta \cdot N(0,1).
\endaligned
\end{equation}
The dynamics for the heights and widths are the same as in the first benchmark \eqref{eq:bench1_dynamics}. The center are rotated based on the rotation matrix $R(\theta_t^{D-1},\dots,\theta_t^1)=R^{D-1}(\theta_t^{D-1})\dots R^1(\theta_t^1)$, where each rotation matrix $R^d(\theta_t^d)$ performs the rotation in the $d$-$(d+1)$ plane by angle $\theta_t^d$.

We handle the technicalities similarly as for the first benchmark. If the variables get out of bounds, we project them back. We initialize the centers randomly in the search space $[x_{\rm min}, x_{\rm max}]^D$. Based on \cite{fu2015robust} the initial heights and widths and generated randomly from their bounds. However, the initial $\theta_1^d$ is set to $\theta_{\rm init}$.

\section{Experimental results}\label{sec:numerics}

In this section, we describe the performance of our methods from Section \ref{sec:algorithms} on the benchmarks from Section \ref{sec:benchmarks}. All displayed results are averaged over $5000$ independent simulations of $\bm\alpha$.

\subsection{Parameter setting}\label{sec:setting}

In Table \ref{table:parameters} we show the used parameters. We first generated the random evolution of $\bm\alpha$ and then uniformly discretized the search space $[x_{\rm min},x_{\rm max}]^D$ into $\neval=2500$ points. Algorithm \ref{alg:method1} randomly selected $2300$ of these $2500$ points at each $t$, evaluted $f(\cdot,\bm\alpha(t))$, selected the best value and invested the remaining $200$ function evaluations into the local search made by the Matlab built-in function \texttt{fmincon}. Algorithm \ref{alg:method2} evaluated all $2500$ points and replaced the function value at a point by the average of all neigboring values with the maximal distance of $3$ (points outside of search space were ignored). The solution with the highest average was deemed to be robust.

\begin{table}[!ht]
\caption{Parameter Values for Benchmark Problems}
\label{table:parameters}
\centering 
\begin{tabular}{@{}lll@{}} \\\toprule
Parameter & Benchmark 1 & Benchmark 2 \\\midrule
$\neval$ & $2500$ & $2500$ \\
$M$ & $5$ & $25$ \\
$D$ & $2$ & $2$ \\
$\lambda$ & \{0,1\} & - \\
$[x_{\rm min}, x_{\rm max}]$ & $[0, 50]$ & $[-25, -25]$ \\
$[h_{\rm min}, h_{\rm max}]$ & $[30, 70]$ & $[30, 70]$ \\
$[w_{\rm min}, w_{\rm max}]$ & $[1, 12]$ & $[1, 13]$ \\
$[\theta_{\rm min}, \theta_{\rm max}]$ & - & $[-\pi, \pi]$ \\
$\sigma_h$ & $U(1, 10)$ & $5$ \\
$\sigma_w$ & $U(0.1, 1)$ & $0.5$ \\
$\sigma_\theta$ & - & $1$ \\
$h_{\rm init}$ & $50$ & $U(h_{\rm min}, h_{\rm max})$ \\
$w_{\rm init}$ & $6$ & $U(w_{\rm min}, w_{\rm max})$ \\
$\theta_{\rm init}$ & - & $0$ \\\bottomrule
\end{tabular}
\end{table}

Even though it is possible to implement predicting future values by using function evaluations at previous time instants, we decided not to do so. The reason is that even this basic method significantly outperforms the state-of-the-art algorithms and adding the predictions could cloud the basic idea.

\subsection{Numerical results}

We compare three methods. {\mA} and {\mB} are based on Algorithm \ref{alg:method1} with the difference that {\mA} does not perform the local search. {\mC} is based on Algorithm \ref{alg:method2}. Numerical details are described in Section \ref{sec:setting}.

We compare the {\mB} method to known results in Table \ref{table:comparison}. On Benchmark 1 with $\lambda\in\{0,1\}$ and Benchmark 2 we show the averaged objective $F_{\rm aver}$ with time window $S\in\{2,6\}$ and the survival function $F_{\rm surv}$ with $\delta\in\{40,50\}$; both defined in \eqref{eq:metrics}. We used the horizon $T=100$ and the results shown are averages for all time instants with $t\in[20,100]$. For all benchmarks and evaluation criteria, our results are significantly better than the best-known results. We comment more on how we collected the best-known results in Appendix \ref{app:comparison}.

We can even show that our results are almost optimal. Consider Benchmark 1 with $\lambda=0$. Discussion in Appendix \ref{app:peak} shows that the optimal solution has the expected value of approximately $65$. Since the peak moves with stepsize $s^m=1$ and the average width is $6.5$, the objective drops to $65-6.5=58.5$ for the next time instant. But this gives the expected objective $\frac12(65+58.5)=61.75$ for $S=2$ to which our value $61.13$ from Table \ref{table:comparison} is very close.

This intuition is confirmed in Table \ref{table:distance} where we show the gap between the optimal objective and the objective found. {\mA} shows approximately half of the theoretical gap \eqref{eq:bound1} while this gap is almost zero when we improve it by the local search via {\mB}. This means that {\mB} found the centre of the highest peak. We would like to stress that the information about the highest peak was not used during the optimization and we used it only a posteriori for evaluating performance.

\begin{table}[!ht]
\caption{Gap Between the Best Possible Objective $f_t^*$ and the Objective Found by Our Methods}
\label{table:distance}
\centering 
\begin{tabular}{@{}lllll@{}} \\\toprule
 & Maximal gap \eqref{eq:bound1} & {\mA} & {\mB} & {\mC} \\\midrule
Benchmark 1 & $4.69$ &$2.18$ & $0.09$ & $5.38$ \\
Benchmark 2 & $5.05$ &$0.99$ & $0.15$ & $4.38$ \\\bottomrule
\end{tabular}
\end{table}

\begin{table*}[!ht]
\caption{Comparison of Best Known and our Results. All Methods Use $2500$ Function Evaluations at Each Time Instant. The Process of Collecting the Best Known Results is Described in Appendix \ref{app:comparison}. All Experiments Were Repeated $5000$ Times.}
\label{table:comparison}
\centering 
\begin{tabular}{@{}ll llll llll@{}} \\\toprule
Setting & From & \multicolumn{4}{c}{Best known result} & \multicolumn{4}{c}{Our result} \\\cmidrule(l{3pt}r{3pt}){3-6}\cmidrule(l{3pt}r{3pt}){7-10}
 & & \multicolumn{2}{c}{$F_{\rm aver}$} & \multicolumn{2}{c}{$F_{\rm surv}$} & \multicolumn{2}{c}{$F_{\rm aver}$} & \multicolumn{2}{c}{$F_{\rm surv}$} \\\cmidrule(l{3pt}r{3pt}){3-4}\cmidrule(l{3pt}r{3pt}){5-6}\cmidrule(l{3pt}r{3pt}){7-8}\cmidrule(l{3pt}r{3pt}){9-10}
& & $S=2$ & $S=6$ & $\delta=40$ & $\delta=50$ & $S=2$ & $S=6$ & $\delta=40$ & $\delta=50$ \\\midrule
Benchmark 1 with $\lambda=1$ & \cite{fu2013finding} & $53.48$ & $8.82$ & $3.02$ & $1.69$ & $63.32$ & $58.76$ & $13.72$ & $10.11$ \\
Benchmark 1 with $\lambda=0$ & \cite{jin2013framework} \cite{yazdani2017new} & - & - & $8.35$ & $4.25$ & $61.13$ & $54.77$ & $10.42$ & $5.91$\\
Benchmark 2 & \cite{fu2015robust} \cite{novoa2018approximation} & $48.88$ & $40.58$ & $1.35$ & $1.02$ & $62.21$ & $57.58$ & $16.54$ & $6.38$ \\\bottomrule
\end{tabular}
\end{table*}


Tables \ref{table:distance} and \ref{table:comparison} also suggest why other methods performed subpar:
\begin{enumerate}
\item Since the {\mB} solution lies in the peak centre, it is a natural candidate for the robust solution as well. We believe that the commonly used particle swarm optimization was far away from the peak centre.
\item While incorporating objective tracking, the previous papers needed to reevaluate the point at previous time instants. This reduced the number of investigated points.
\end{enumerate}
Note that as explained at the end of Section \ref{sec:algorithms}, our methods do not suffer from these problems.

We show additional results for Benchmarks 1 and 2 in Figures \ref{fig:bench1} and \ref{fig:bench2}, respectively. For Benchmark 1 the columns show the results for $\lambda=0$ (left) and $\lambda=1$ (right) while for Benchmark 2 the columns show the random generation of initial centers (left) or the grid generation described in Appendix \ref{app:comparison} (right). We can observe the following phenomena:
\begin{itemize}\itemsep 0pt
\item The method with local search {\mB} outperforms the method without the local search {\mA} in all cases.
\item The survival time for {\mC} is better than for {\mB} only for one benchmark.
\item The survival time is stable for Benchmark 2 while it increases with increasing time for Benchmark 1. The reason is that Benchmark 1 initializes the peak heights to $50$ while Benchmark 2 initializes them randomly in $[30,70]$. Thus, for the former case, the maximal peak height is much smaller for the initial time instants.
\item The initialization or parameters have a large impact on the solution (comparison of left and right columns).
\item Benchmark 2 is not affected by the boundary conditions for variables. This does not hold for Benchmark 1 where the survival time increases as the centres hit the boundary and stay there.
\end{itemize}

\noindent To summarize, the {\mB} method, which does not utilize any tracking or future predictions, performs very well on both benchmarks. This raises the question of whether the moving benchmark problem is suitable for ROOT.

\section{Conclusion}

In this paper, we gave a proper description of the moving benchmark problem for ROOT and proposed a simple method to solve it. Our method significantly outperforms other methods. Since we believe that there are multiple deficiencies in most ROOT papers, we suggest that the papers on ROOT should include the following information to facilitate further comparisons and analyses of proposed algorithms:
\begin{enumerate}
\item \textbf{Proper problem description}. Including parameters, special setting and initial conditions. This is needed for other authors to repeat the experiments.
\item \textbf{Codes available online}. When it is not possible to describe everything, codes online help significantly.
\item \textbf{Fair comparison}. In some papers, a comparison was done with different parameter setting. Including higher computational budget.
\item \textbf{Higher number of repetitions}. When the experiment is repeated $20$ or $30$ times as in most papers, the graphs are not smooth and it may be difficult to extract useful information from them.
\item \textbf{Comparison with a basic method}. Sometimes a simple solution (centre of the highest peak) performs well in a more complicated setting (robust solution).
\end{enumerate}
Note that most papers investigated in this manuscript violated all these topics mentioned above.

\appendix

In the Appendix, we provide further technical results that support the main text.

\subsection{Estimate on solution quality}\label{app:lemma}

We recall first two definitions. We say that a function $g$ is Lipschitz on $X$ with constant $L$ if    
$$
\nrm{g(\bm x) - g(\bm y)}\le L \norm{\bm x-\bm y}
$$
for all $\bm x,\bm y\in X$. We say that $\{\bm x^1,\dots,\bm x^S\}$ is $\delta x$-cover of $X$ if for each $\bm x\in X$ there is some $s\in\{1,\dots,S\}$ such that $\norm{\bm x-\bm x^s}\le\delta x$. Then we have the following lemma.

\begin{lemma}\label{lemma}
Consider an optimization problem
\begin{equation}\label{eq:lemma}
\mxmz_{\bm x\in X}\quad g(\bm x),
\end{equation}
where $g$ is Lipschitz continuous with constant $L$. Denote $\bm x^1,\dots,\bm x^S$ to be a $\delta x$-cover of $X$ and $\hat {\bm x}\in\argmax_{s=1,\dots,S} g(\bm x^s)$ to be the best sampled value. Then $\hat{\bm x}$ is an $\eps$-optimal solution of \eqref{eq:lemma} in the sense of
$$
g(\hat{\bm x})\ge \sup_{x\in X}g(\bm x) - L\cdot\delta x.
$$ 
\end{lemma}
\begin{proof}
The existence of the $\delta x$-cover and the Lipschitz continuity of $g$ imply that $g$ is bounded from above on $X$. That means that there is a sequence $\{\bm y^n\}_{n=1}^\infty\subset X$ satisfying
\begin{equation}\label{eq:lemma1}
g(\bm y^n)\ge \sup_{x\in X} g(\bm x) - \frac1n.
\end{equation}
Due to the definition of $\delta x$-cover, for each $n$ there is some $s(n)\in\{1,\dots,S\}$ such that $\norm{\bm x^{s(n)} - \bm y^n}\le\delta x$. This implies
$$
\aligned
\max_{s=1,\dots,S}g(\bm x^s) &\ge g(\bm x^{s(n)}) = g(\bm x^{s(n)}) - g(\bm y^n) + g(\bm y^n) \\
&\ge g(\bm y^n) - L\delta x \ge \sup_{x\in X} g(\bm x) - \frac1n - L\delta x,
\endaligned
$$
where the second inequality follows from the Lipschitz continuity of $g$ and the last inequality from \eqref{eq:lemma1}. Since $n$ is arbitrary, the lemma statement follows.
\end{proof}

To apply this to \eqref{eq:bound1}, it suffices to realize that uniform sampling with $N$ points form a $\delta x$-cover for $[x_{\rm min},x_{\rm max}]^D$ with
$$
\delta x = \frac{\sqrt{D}(x_{\max}-x_{\min})}{2(N^{\frac 1D}-1)}.
$$

\subsection{Differences in benchmark problems from other papers}\label{app:dynamics}

In this section, we comment on small details in the benchmark description. All the mentioned papers wrote $\bm r$ instead of $\bm r_t^m$ in \eqref{eq:bench1_dynamics}. However, since they commented on random movement, we believe that the time-dependence has to be stressed because otherwise, the centres would move in a fixed direction.

The complete problem description also includes what happens when peak height, weight or centre get outside the allowed boundary. While some of the paper described that they are projected back onto the boundary, \cite{fu2012characterizing} noted that they are ``bounced back'', most of the papers did not describe what happens in such a situation. However, this may have a huge impact on the solution. 

Most of the papers generated the initial random vector $\bm r_1^m$ by generating all components randomly in $[-1,1]$ and then normalized the vector into the length of $s^m$. However, this is not equivalent to generating randomly on the sphere with a radius of $s^m$. Figure \ref{fig:counts} shows the angle between the generated vector and the vector $(1,0)$ in the two-dimensional case. The approach from the earlier papers gives a much higher chance for the (normalized) vectors around $(\pm1,\pm1)$. The reason is that the square is ``bigger'' than the circle in these directions.

Finally, \cite{fu2015robust} initialized the initial centres of $25$ peaks by selecting $5$ random points in each dimension and then performing Cartesian product. As we show in Figure \ref{fig:bench2}, this yields hugely different results from randomly generating in the domain.

\begin{figure}[!ht]
\centering
\begin{tikzpicture}
 \pgfplotsset{small,width=8cm,samples=30}
 \begin{groupplot}[xmin=0, xmax=6.2832, group style = {group size = 1 by 1}, grid=major, grid style={dotted, gray!50}]
 \nextgroupplot[xlabel={Angle}, ylabel={Density}, xtick={0,1.5708,3.1416,4.7124,6.2832}, xticklabels={$0$,$\frac12\pi$,$\pi$,$\frac32\pi$,$2\pi$}, legend style = {column sep=10pt, legend columns=2, font=\footnotesize, legend to name=grouplegendA}]
    \addplot [line1] table[x index=0, y index=1] {\tabCounts}; \addlegendentry{Sampling on a square};
    \addplot [line2] table[x index=0, y index=2] {\tabCounts}; \addlegendentry{Uniform sampling};
 \end{groupplot}
  \node at ($(group c1r1) + (0,3.2)$) {\ref{grouplegendA}};  
\end{tikzpicture}
\caption{The sampling on a square, used in previous papers, does not result in uniform sampling.}
\label{fig:counts}
\end{figure}
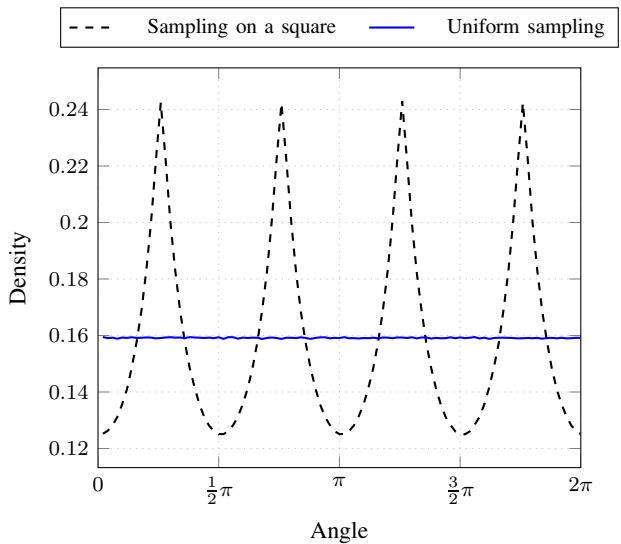

Finally, \cite{yazdani2017new} used a different function count. While the original and our approach recomputed the solution at every time step and then computed its survival based on the future values \cite{yazdani2017new} recomputed the solution only when it dropped below the threshold $\delta$. This resulted in the fact that they used approximately $8$ times more function evaluations.

\subsection{Selecting the best known results}\label{app:comparison}

In this section we describe how we collected the best known results from Table \ref{table:comparison}. Benchmark 1 with $\lambda=1$ is taken from \cite{fu2013finding}, Benchmark 1 with $\lambda=0$ from \cite{yazdani2017new} and Benchmark 2 from \cite{fu2015robust}. Note that \cite{yazdani2017new} compared himself with the results from \cite{fu2013finding,jin2013framework,guo2014find} and showed that their results are superior. For Benchmark 2 we considered only the random movement which in \cite{fu2015robust} was denoted as $TP_{13}$. Finally we did not compare ourself to the promising-looking results from \cite{yazdani2018robust} because they used different setting for the stepsize $s$.

Note that due to the issues described earlier, it may have happened that the setting for our and their papers is different. However, we tried to minimize this possibility.

\subsection{Height of the heighest peak}\label{app:peak}

In Figure \ref{fig:max} we intitialize $M$ peaks with initial heights $h_1^m=h_{\rm init}=50$. We apply the dynamics \eqref{eq:bench1_dynamics} and observe the average height of the highest peak for time instants $t\in[1,20]$. We see that rather soon the average height stabilizes at $65$ for $M=5$ and close to the maximal value $h_{\rm max}=70$ for $M=25$. This is the optimal value for $F_{\rm aver}$ for $S=1$.

\begin{figure}[!ht]
\centering
\begin{tikzpicture}
 \pgfplotsset{small,width=8cm,samples=30}
 \begin{groupplot}[xmin=0, xmax=21, group style = {group size = 1 by 1}, grid=major, grid style={dotted, gray!50}]
 \nextgroupplot[xlabel={Time $t$}, ylabel={Maximal peak height}, legend style = {align=left, column sep=10pt, font=\footnotesize}, legend pos={south east}, legend cell align={left}]
    \addplot [line1] table[x index=0, y index=1] {\tabMax}; \addlegendentry{$M=5$};
    \addplot [line2] table[x index=0, y index=2] {\tabMax}; \addlegendentry{$M=25$};
 \end{groupplot}
\end{tikzpicture}
\caption{The average height of the heighest of $M$ peaks.}
\label{fig:max}
\end{figure}
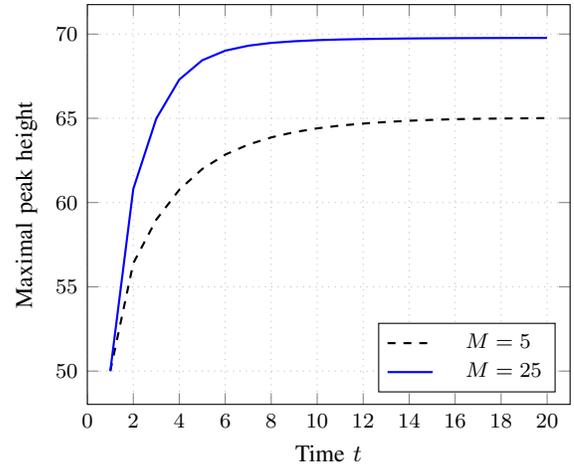

\begin{figure*}[!ht]
\begin{tikzpicture}
 \pgfplotsset{small,width=0.5\linewidth,samples=30}
 \begin{groupplot}[group style = {group size = 2 by 3, vertical sep = 5pt, horizontal sep = 24pt}, grid=major, grid style={dotted, gray!50}, legend cell align={left}]
 \nextgroupplot[row1, align=center, title={$\lambda=0$\\ \\}, font=\footnotesize, ylabel={$F_{\rm aver}$ averaged objective}, xlabel={Averaging window length $S$}, legend style = {column sep=10pt, legend columns=4, legend to name=grouplegendB}]
    \addplot [line2a] table[x index=0, y index=1] {\tabBenchAAverA};\addlegendentry{{\mA}};        
    \addplot [line3a] table[x index=0, y index=2] {\tabBenchAAverA};\addlegendentry{{\mB}};        
    \addplot [line4a] table[x index=0, y index=3] {\tabBenchAAverA};\addlegendentry{{\mC}};        
 \nextgroupplot[row1, align=center, title={$\lambda=1$\\ \\}, font=\footnotesize, ylabel={}, xlabel={Averaging window length $S$}]
    \addplot [line2a] table[x index=0, y index=1] {\tabBenchAAverB};        
    \addplot [line3a] table[x index=0, y index=2] {\tabBenchAAverB};        
    \addplot [line4a] table[x index=0, y index=3] {\tabBenchAAverB};        
  \nextgroupplot[row2, align=center, font=\footnotesize, ylabel={$F_{\rm surv}$ survival time for $\delta=40$}, xticklabels={}, yshift=-2.5em]
    \addplot [line2a] table[x index=0, y index=1] {\tabBenchASurvAA};
    \addplot [line3a] table[x index=0, y index=2] {\tabBenchASurvAA};
    \addplot [line4a] table[x index=0, y index=3] {\tabBenchASurvAA};
 \nextgroupplot[row2, align=center, font=\footnotesize, ylabel={}, xticklabels={}, yshift=-2.5em]
    \addplot [line2a] table[x index=0, y index=1] {\tabBenchASurvBA};        
    \addplot [line3a] table[x index=0, y index=2] {\tabBenchASurvBA};        
    \addplot [line4a] table[x index=0, y index=3] {\tabBenchASurvBA};
  \nextgroupplot[row2, align=center, font=\footnotesize, ylabel={$F_{\rm surv}$ survival time for $\delta=50$}, xlabel={Starting time}, yshift=-0.5em]
    \addplot [line2a] table[x index=0, y index=1] {\tabBenchASurvAB};
    \addplot [line3a] table[x index=0, y index=2] {\tabBenchASurvAB};
    \addplot [line4a] table[x index=0, y index=3] {\tabBenchASurvAB};
 \nextgroupplot[row2, align=center, font=\footnotesize, ylabel={}, xlabel={Starting time}, yshift=-0.5em]
    \addplot [line2a] table[x index=0, y index=1] {\tabBenchASurvBB};        
    \addplot [line3a] table[x index=0, y index=2] {\tabBenchASurvBB};        
    \addplot [line4a] table[x index=0, y index=3] {\tabBenchASurvBB};
 \end{groupplot}
 \node at ($(group c1r1) + (4,3.6)$) {\ref{grouplegendB}};  
\end{tikzpicture}
\caption{Results for Benchmark 1 with $\lambda=0$ (left) and $\lambda=1$ (right). We show the averaged objective $F_{\rm aver}$ as a function of the averaging time window $S$ (top) and the survival function $F_{\rm surv}$ for thresholds $\delta=40$ (middle) and $\delta=50$ (bottom). Note that the metrics are defined in \eqref{eq:metrics}.}
\label{fig:bench1}
\end{figure*}
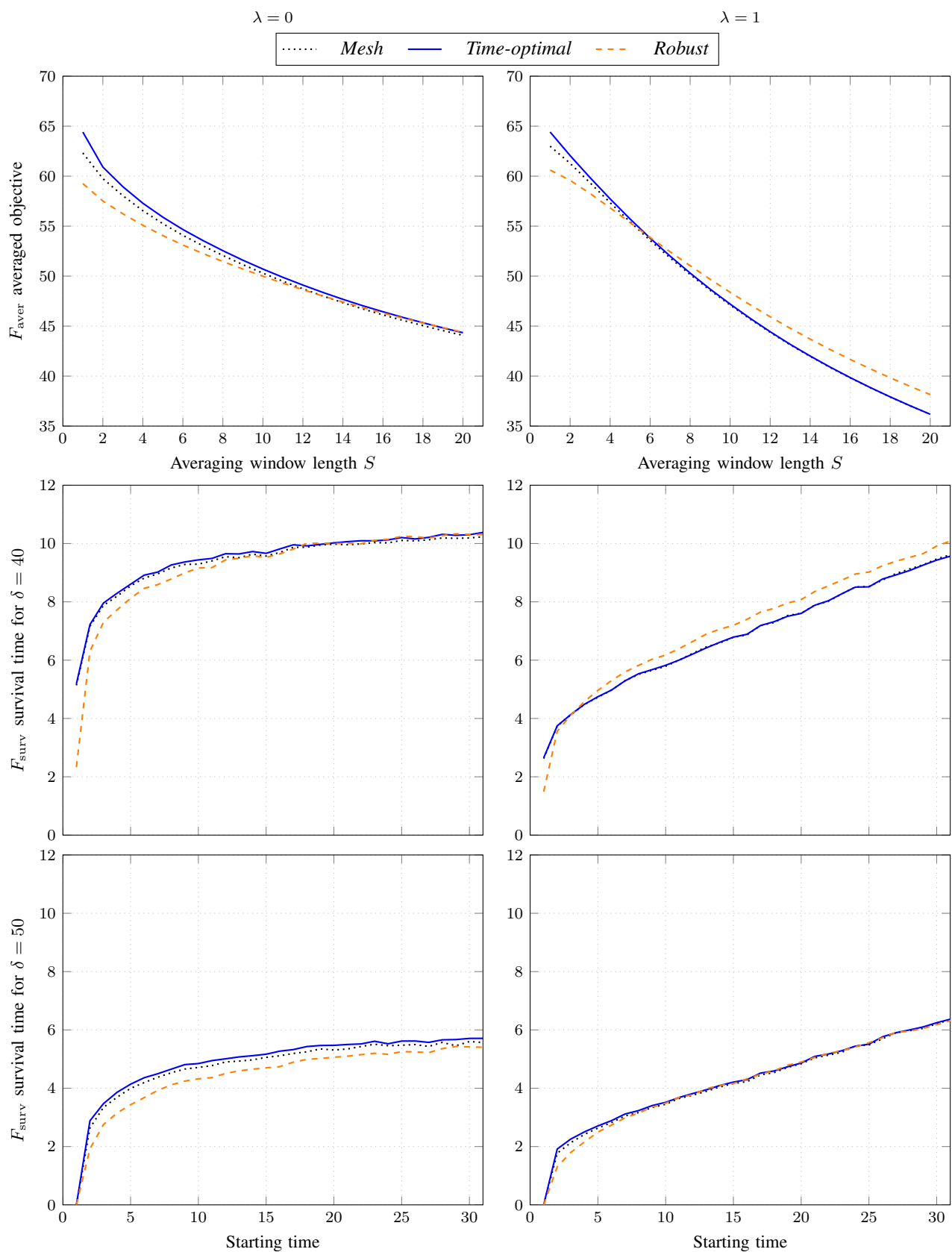

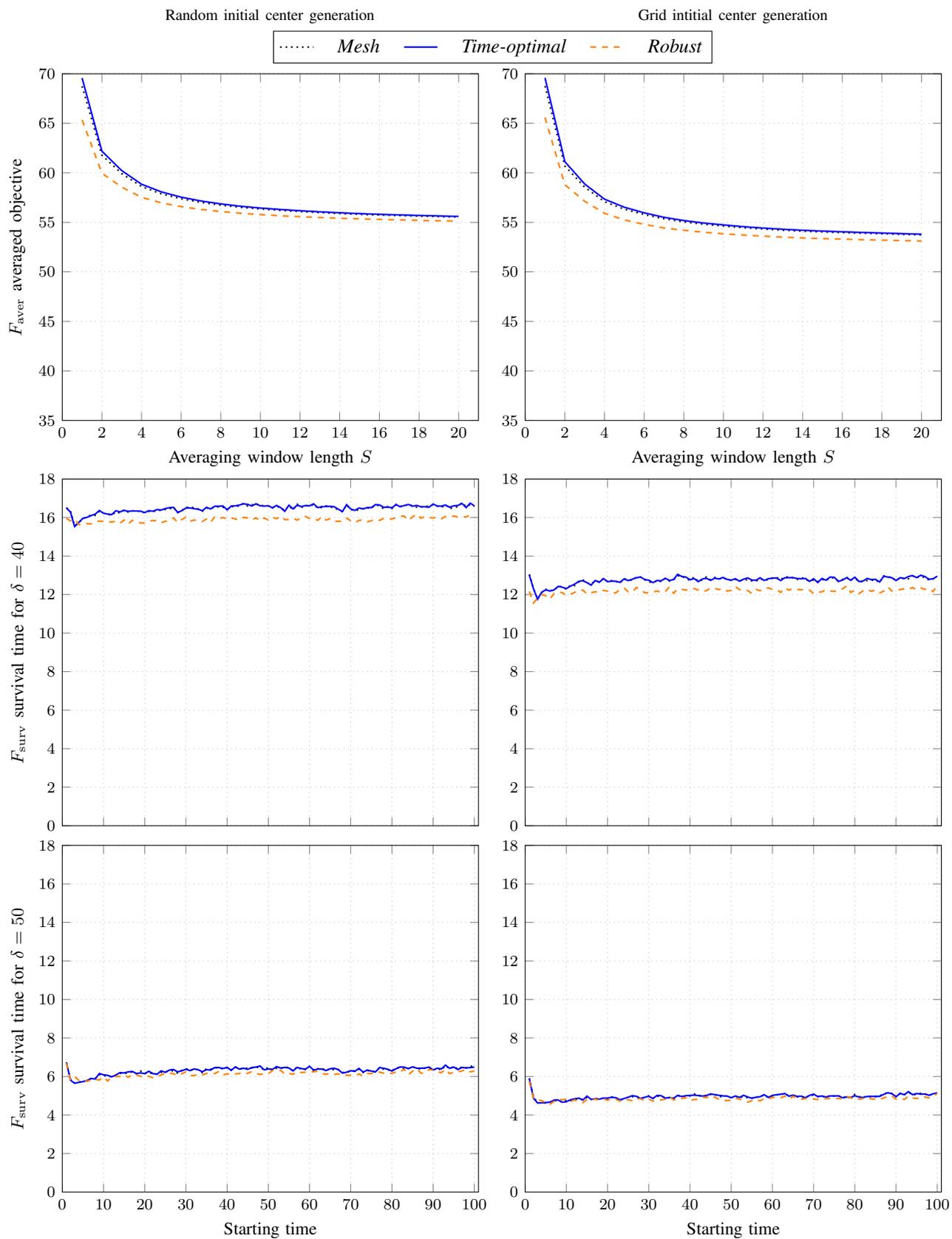
\begin{figure*}[!ht]
\begin{tikzpicture}
 \pgfplotsset{small,width=0.5\linewidth,samples=30}
 \begin{groupplot}[group style = {group size = 2 by 3, vertical sep = 5pt, horizontal sep = 24pt}, grid=major, grid style={dotted, gray!50}, legend cell align={left}]
 \nextgroupplot[row1, align=center, title={Random initial center generation \\ \\}, font=\footnotesize, ylabel={$F_{\rm aver}$ averaged objective}, xlabel={Averaging window length $S$}, legend style = {column sep=10pt, legend columns=4, legend to name=grouplegendC}]
    \addplot [line2a] table[x index=0, y index=1] {\tabBenchBAverA};\addlegendentry{{\mA}};        
    \addplot [line3a] table[x index=0, y index=2] {\tabBenchBAverA};\addlegendentry{{\mB}};        
    \addplot [line4a] table[x index=0, y index=3] {\tabBenchBAverA};\addlegendentry{{\mC}};        
 \nextgroupplot[row1, align=center, title={Grid intitial center generation\\ \\}, font=\footnotesize, ylabel={}, xlabel={Averaging window length $S$}]
    \addplot [line2a] table[x index=0, y index=1] {\tabBenchBAverB};        
    \addplot [line3a] table[x index=0, y index=2] {\tabBenchBAverB};        
    \addplot [line4a] table[x index=0, y index=3] {\tabBenchBAverB};        
  \nextgroupplot[row3, align=center, font=\footnotesize, ylabel={$F_{\rm surv}$ survival time for $\delta=40$}, xticklabels={}, yshift=-2.5em]
    \addplot [line2a] table[x index=0, y index=1] {\tabBenchBSurvAA};
    \addplot [line3a] table[x index=0, y index=2] {\tabBenchBSurvAA};
    \addplot [line4a] table[x index=0, y index=3] {\tabBenchBSurvAA};
 \nextgroupplot[row3, align=center, font=\footnotesize, ylabel={}, xticklabels={}, yshift=-2.5em]
    \addplot [line2a] table[x index=0, y index=1] {\tabBenchBSurvBA};        
    \addplot [line3a] table[x index=0, y index=2] {\tabBenchBSurvBA};        
    \addplot [line4a] table[x index=0, y index=3] {\tabBenchBSurvBA};
  \nextgroupplot[row3, align=center, font=\footnotesize, ylabel={$F_{\rm surv}$ survival time for $\delta=50$}, xlabel={Starting time}, yshift=-0.5em]
    \addplot [line2a] table[x index=0, y index=1] {\tabBenchBSurvAB};
    \addplot [line3a] table[x index=0, y index=2] {\tabBenchBSurvAB};
    \addplot [line4a] table[x index=0, y index=3] {\tabBenchBSurvAB};
 \nextgroupplot[row3, align=center, font=\footnotesize, ylabel={}, xlabel={Starting time}, yshift=-0.5em]
    \addplot [line2a] table[x index=0, y index=1] {\tabBenchBSurvBB};        
    \addplot [line3a] table[x index=0, y index=2] {\tabBenchBSurvBB};        
    \addplot [line4a] table[x index=0, y index=3] {\tabBenchBSurvBB};
 \end{groupplot}
 \node at ($(group c1r1) + (4,3.6)$) {\ref{grouplegendB}};  
\end{tikzpicture}
\caption{Results for Benchmark 2 with random center generation (left) and the grid center generation described in Appendix \ref{app:comparison} (right). We show the averaged objective $F_{\rm aver}$ as a function of the averaging time window $S$ (top) and the survival function $F_{\rm surv}$ for thresholds $\delta=40$ (middle) and $\delta=50$ (bottom). Note that the metrics are defined in \eqref{eq:metrics}.}
\label{fig:bench2}
\end{figure*}


\bibliographystyle{IEEEtran}
\bibliography{Bibliography}

\end{document}